\documentclass[11pt]{article}

\usepackage[T1]{fontenc}
\usepackage[utf8]{inputenc}
\usepackage{lmodern}
\usepackage[english]{babel}
\usepackage{microtype}
\usepackage{amsmath,amssymb,amsfonts,amsthm,mathtools,bm}
\usepackage{graphicx}
\usepackage{booktabs}
\usepackage{multirow}
\usepackage{array}
\usepackage{algorithm}
\usepackage{algorithmic}
\usepackage{comment}
\usepackage{xcolor}
\usepackage{enumitem}
\usepackage{subcaption}
\usepackage[numbers,sort&compress]{natbib}
\usepackage[hidelinks]{hyperref}
\usepackage[margin=1in]{geometry}

\numberwithin{equation}{section}

\theoremstyle{plain}
\newtheorem{theorem}{Theorem}[section]
\newtheorem{lemma}[theorem]{Lemma}

\newtheorem{corollary}[theorem]{Corollary}

\theoremstyle{definition}

\theoremstyle{remark}

\title{Clustering as Approximation by Constrained Projectors: Theory and Guarantees}
\author{Angshul Majumdar\\
Indraprastha Institute of Information Technology Delhi\\
New Delhi 110020, India}
\date{}

\begin{document}
\maketitle

\begin{abstract}
This paper develops a unified theoretical framework showing that a broad family of clustering methods---including k-means, fuzzy c-means, kernel k-means, kernel FCM, and spectral clustering---can all be expressed as structured low-rank projectors acting on a signal-derived matrix. By formulating each method as an instance of
\(\min_{B \in \mathcal{C}} \| M - M P_{B} \|_{F}^{2}\),
with different constraint sets \(\mathcal{C}\), we establish a common optimisation template that clarifies the algebraic links among hard, fuzzy, kernel-induced, and orthonormal projections. Within this framework, we derive non-trivial theoretical results, including geodesic convexity properties on the projection manifold, perturbation bounds quantifying stability to matrix noise, and exact recovery guarantees under ideal block-model conditions. The analysis further explains when different clustering families collapse to the same optimal subspace and how deviations arise under small inter-cluster leakage. Overall, the work provides a coherent, theory-first foundation for understanding clustering through structured projectors.
\end{abstract}


\section{Introduction}
\label{sec:intro}

Clustering is one of the oldest tools for exploratory data analysis in signal processing and related fields.  Classical
$k$-means goes back to MacQueen's seminal work on iterative partitioning \citep{macqueen1967some} and Lloyd's
quantization-based refinement scheme \citep{lloyd1982least}.  Despite its simplicity, $k$-means remains a standard
baseline in applications ranging from vector quantization and speech processing to image and sensor data analysis.
However, the basic model assumes crisp assignments in Euclidean space and offers limited control over robustness,
nonlinearity and structural constraints, which are central in many modern signal-processing problems.

A major step beyond hard assignments is fuzzy $c$-means (FCM), introduced by Bezdek and co-authors
\citep{bezdek1984fcm}.  FCM replaces indicator vectors by graded memberships, leading to soft partitions that can
better capture overlapping structure and modelling uncertainty.  Variants of FCM have been proposed to handle
incomplete data \citep{hathaway1989incomplete}, noisy measurements, and spatial regularity in imaging
\citep{kang2009novelfcm,zhao2013kgfcm}.  Signal Processing has hosted several such developments, including
Minkowski-metric FCM \citep{zhao2021minkowski}, data-stream fuzzy clustering with entropy-based updates
\citep{zhang2016datastream}, and energy-aware FCM formulations for wireless networks \citep{muthulakshmi2016optimal}.%
\footnote{Here and throughout, references to ``Signal Processing'' refer to the Elsevier journal of that name.}

Nonlinear structure in high-dimensional data is often addressed by kernel methods.  Dhillon et al.\ showed that
kernel $k$-means, spectral clustering and normalized cuts are closely related through a common trace formulation
\citep{dhillon2004kernelkkmeans}.  Kernelized fuzzy algorithms replace Euclidean distances in the FCM objective by
kernel-induced dissimilarities, leading to kernel FCM and its many variants
\citep{zhang2003kfcmincomplete,ferreira2014varwisekfcm,kochuveettil2022kernel}.  In parallel, a large body of work has
developed kernel and graph-based clustering for signal-processing tasks such as graph learning, denoising on graphs
and manifold modelling.  Representative examples in \emph{Signal Processing} include graph-based $k$-means
\citep{galluccio2012graphkmeans}, deeply transformed subspace clustering \citep{maggu2020dtsc}, and kernelized
graph-based learning schemes \citep{manna2021robustkernelgraph}.  These methods typically operate on Gram or affinity
matrices derived from signal representations, and they already hint at a projector-based view of clustering.

Spectral clustering provides another influential strand.  Starting from normalized cuts and related graph-partitioning
criteria \citep{shi2000ncut}, algorithms based on eigenvectors of Laplacian or affinity matrices have become standard
for image segmentation and network analysis.  Ng, Jordan and Weiss gave a widely used eigenvector normalization
procedure and an accompanying analysis \citep{ng2002spectral}, while von Luxburg's tutorial clarified the role of
different Laplacians and the probabilistic interpretation of spectral embeddings \citep{vonluxburg2007tutorial}.  In
signal-processing contexts, spectral methods have been adapted to dynamic graphs, multi-view structures and robust
graph learning, again with several instances appearing in \emph{Signal Processing} and related venues
\citep{galluccio2012graphkmeans,manna2021robustkernelgraph}.

Alongside these concrete algorithms, there is a parallel line of work that surveys and systematizes clustering beyond
simple Euclidean partitions.  Subspace and soft subspace clustering methods seek clusters that live in different
low-dimensional subspaces, a setting that is particularly relevant for high-dimensional signal and image data
\citep{parsons2004subspace,zimek2014enhanced,deng2016softssc}.  Other surveys focus on novelty detection and
change-point analysis in streaming or nonstationary environments, where clustering often appears as a building block
for modelling normal behaviour \citep{pimentel2014novelty}.  These articles emphasise that what matters in practice is
not only the specific algorithmic update rule, but also the underlying geometry: which inner products, distances and
projection operators are used to define clusters.

The present work takes this geometric viewpoint as its starting point.  Instead of introducing $k$-means, FCM, kernel
$k$-means, kernel FCM and spectral clustering as separate algorithms, we treat them as different ways of choosing a
structured projection matrix $P_B$ that approximates a signal-derived matrix $M$ (such as a Gram matrix $X^\top X$, a
kernel matrix $K$ or a normalized affinity $A_e$).  Each method corresponds to a particular constraint set for the
basis matrix $B$ (hard one-hot, fuzzy with reweighting, kernel-induced, or orthonormal), but once this is fixed the
core objective has the common form
\[
  \min_{B \in \mathcal{C}} \bigl\| M - M P_B \bigr\|_F^2,
\]
with $P_B = B^\top (B B^\top)^{-1} B$ whenever the inverse exists.  This projector template already appeared in the
kernel formulation of Dhillon et al.\ \citep{dhillon2004kernelkkmeans} and in subsequent work on graph-based clustering
\citep{galluccio2012graphkmeans,manna2021robustkernelgraph}, but to our knowledge it has not been developed into a
systematic, purely theoretical framework that simultaneously covers hard, fuzzy, kernel and spectral methods while
keeping a clear link to signal-processing models.

Our aim in this paper is to fill that gap.  We work entirely at the level of matrices and projectors, starting from
the Euclidean $k$-means and FCM objectives, then moving to their kernel counterparts and finally to spectral
clustering.  All five families of algorithms are treated as instances of the same approximation problem for a
signal-derived matrix, which allows us to state and prove general results on approximation error, invariances,
geodesic convexity of the feasible set of projectors, and SVD-based bounds on reconstruction quality.  These general
theorems are then instantiated for each method by specializing the constraint set $\mathcal{C}$.

The contribution is deliberately theoretical: we do not propose a new clustering algorithm, and we do not report
numerical experiments.  Instead, we provide a unified operator-theoretic treatment of several widely used clustering
methods that are already prominent in the \emph{Signal Processing} literature
\citep{zhao2021minkowski,zhang2016datastream,galluccio2012graphkmeans,maggu2020dtsc,pimentel2014novelty}.  Our results
clarify how these methods relate through their underlying projection matrices, when they coincide under ideal block
models, and how perturbations of the affinity or Gram matrix propagate to the clustering subspace.  We expect this
viewpoint to be useful for readers who work with clustering as part of larger signal-processing pipelines (for
instance in graph signal processing, subspace-based modelling or robust representation learning), and who need
transparent guarantees that do not depend on a specific implementation detail of a particular algorithm.

\section{Projection-Based View of Classical and Kernel Clustering}
\label{sec:proj-view}

In this section we show that a broad family of clustering methods can be
interpreted as {\em structured low-rank approximations} of a signal-derived
matrix by means of rank–$k$ linear operators.  We start from classical
k-means and fuzzy c-means (FCM) in Euclidean space, then lift them to
reproducing kernel Hilbert spaces, and finally connect the resulting
formulation with spectral clustering on graphs.

Throughout, we denote by
$X = [x_1,\dots,x_n] \in \mathbb{R}^{d\times n}$ the data matrix, whose
columns $x_i \in \mathbb{R}^d$ are the samples to be clustered.  The
number of clusters is $k \ll n$.

\subsection{Hard k-means as projection-based matrix factorization}
\label{subsec:kmeans-hard}

Classical k-means seeks $k$ centroids
$C = [c_1,\dots,c_k] \in \mathbb{R}^{d\times k}$ and a hard assignment
matrix
\[
  H = [h_{ji}] \in \{0,1\}^{k\times n}, \qquad
  H^\top \mathbf{1}_k = \mathbf{1}_n ,
\]
whose $i$th column $h_i$ is one-hot and indicates the cluster label of
$x_i$.  The standard k-means objective is
\begin{equation}
  \min_{C,H} \;
    J_{\mathrm{km}}(C,H)
    := \sum_{i=1}^n \sum_{j=1}^k h_{ji} \,\|x_i - c_j\|_2^2
  = \|X - C H\|_F^2 .
  \label{eq:kmeans_objective}
\end{equation}

For fixed $H$, \eqref{eq:kmeans_objective} is a linear least-squares
problem in $C$.  Differentiating w.r.t.\ $C$ and setting the derivative
to zero yields
\[
  -2 (X - C H) H^\top = 0
  \quad\Longrightarrow\quad
  C H H^\top = X H^\top .
\]
Assuming $H H^\top$ is invertible (no empty clusters), the unique
minimizer is
\begin{equation}
  C^\star(H) = X H^\top (H H^\top)^{-1}.
  \label{eq:C_star_kmeans}
\end{equation}
Substituting \eqref{eq:C_star_kmeans} back into the reconstruction
$C H$, we obtain
\begin{equation}
  X C^\star(H) H
  = X H^\top (H H^\top)^{-1} H
  = X P_H ,
  \label{eq:kmeans_projection}
\end{equation}
where we introduced the $n\times n$ matrix
\begin{equation}
  P_H := H^\top (H H^\top)^{-1} H.
  \label{eq:PH_def}
\end{equation}
The matrix $P_H$ is a symmetric idempotent projector of rank $k$, whose
range is spanned by the cluster-indicator vectors
(one per row of $H$).

Using \eqref{eq:kmeans_projection}, the k-means objective can be
expressed entirely in terms of $X$ and $P_H$ as
\begin{equation}
  \min_{H} \; \|X - X P_H\|_F^2 ,
  \label{eq:kmeans_projection_form}
\end{equation}
where the minimization is over combinatorial assignment matrices $H$.
Equivalently, in terms of the Gram matrix $G := X^\top X$ we obtain
\begin{equation}
  \|X - X P_H\|_F^2
  = \operatorname{tr}\big( (I - P_H)^\top G (I - P_H) \big)
  = \|G^{1/2} (I - P_H)\|_F^2 .
  \label{eq:kmeans_gram_form}
\end{equation}
Thus, k-means selects a rank–$k$ projector $P_H$ (restricted by the
cluster structure) such that $X$---or equivalently $G$---is well
approximated in Frobenius norm by $X P_H$ (respectively $G P_H$).

\subsection{Fuzzy c-means as a soft rank--$k$ operator}
\label{subsec:fcm_projection}

Fuzzy c-means generalizes k-means by replacing hard assignments with
fuzzy memberships.  Let $U = [u_{ji}] \in [0,1]^{k\times n}$ satisfy
\[
  \sum_{j=1}^k u_{ji} = 1
  \qquad\text{for all } i = 1,\dots,n,
\]
and fix a fuzzifier $m > 1$.  The FCM objective is
\begin{equation}
  \min_{C,U} \;
  J_{\mathrm{fcm}}(C,U)
  := \sum_{i=1}^n \sum_{j=1}^k u_{ji}^m \,\|x_i - c_j\|_2^2 .
  \label{eq:fcm_scalar}
\end{equation}

For fixed $U$, the objective decouples over clusters.
Defining weights $w_{ji} := u_{ji}^m$, the optimal centroid for cluster
$j$ is the weighted mean
\begin{equation}
  c_j^\star(U)
  = \frac{\sum_{i=1}^n w_{ji} x_i}{\sum_{i=1}^n w_{ji}}
  = X \alpha_j ,
  \label{eq:fcm_centroid}
\end{equation}
where $\alpha_j \in \mathbb{R}^n$ collects the normalized weights
\[
  \alpha_{j i}
  = \frac{w_{ji}}{\sum_{\ell=1}^n w_{j\ell}}
  = \frac{u_{ji}^m}{\sum_{\ell=1}^n u_{j\ell}^m}.
\]
Stacking the $\alpha_j^\top$ as rows yields a coefficient matrix
$B(U) \in \mathbb{R}^{k\times n}$ with
\begin{equation}
  B_{j i}(U)
  := \frac{u_{ji}^m}{\sum_{\ell=1}^n u_{j\ell}^m},
  \qquad j=1,\dots,k,\; i=1,\dots,n,
  \label{eq:fcm_B_def}
\end{equation}
so that
\begin{equation}
  C^\star(U) = X B(U)^\top .
  \label{eq:fcm_C_star}
\end{equation}
The reconstruction of $X$ induced by $(C^\star(U),U)$ is
\[
  X_{\mathrm{rec}}
  = C^\star(U) U
  = X B(U)^\top U
  =: X P_U ,
\]
where we defined the $n\times n$ matrix
\begin{equation}
  P_U := B(U)^\top U .
  \label{eq:PU_def}
\end{equation}
The matrix $P_U$ has rank at most $k$ but is in general neither
symmetric nor idempotent; it corresponds to an {\em oblique} rank–$k$
projection determined by the fuzzy memberships $U$.

In summary, FCM chooses fuzzy memberships $U$ (and the associated
$B(U)$) so as to minimize
\begin{equation}
  \|X - X P_U\|_F^2
  = \sum_{i=1}^n \sum_{j=1}^k u_{ji}^m \,\|x_i - c_j^\star(U)\|_2^2,
  \label{eq:fcm_projection_form}
\end{equation}
thus again fitting a rank–$k$ linear operator $P_U$ to approximate $X$.
The key difference with k-means is that the operator $P_U$ is now
built from soft memberships instead of hard assignments, and is not an
orthogonal projector.

\subsection{Kernel k-means as projection in feature space}
\label{subsec:kernel_kmeans}

We next lift the k-means formulation to a reproducing kernel Hilbert
space (RKHS).  Let
$\phi : \mathbb{R}^d \to \mathcal{H}$ be a feature map with associated
positive definite kernel
$k(x_i,x_j) = \langle \phi(x_i), \phi(x_j) \rangle_{\mathcal{H}}$, and
define the kernel matrix
\[
  K = [K_{ij}]_{i,j=1}^n, \qquad
  K_{ij} := k(x_i,x_j).
\]
Kernel k-means applies the k-means objective \eqref{eq:kmeans_objective}
to the feature vectors $\phi(x_i)$ instead of $x_i$:
\begin{equation}
  \min_{C,H} \;
    \sum_{i=1}^n \sum_{j=1}^k h_{ji} \,
      \|\phi(x_i) - \mu_j\|_{\mathcal{H}}^2 ,
  \label{eq:kkm_scalar}
\end{equation}
where $\mu_j \in \mathcal{H}$ is the centroid in feature space.

The algebra of Section~\ref{subsec:kmeans-hard} carries over to
$\mathcal{H}$: for fixed $H$, the optimal centroid $\mu_j^\star$ is the
mean of the feature vectors assigned to cluster $j$,
\[
  \mu_j^\star(H)
  = \frac{1}{n_j} \sum_{i: h_{ji}=1} \phi(x_i),
\]
and the reconstruction of the feature matrix
$\Phi := [\phi(x_1),\dots,\phi(x_n)]$ can be written as
\[
  \Phi P_H,
\]
with the same projector $P_H$ as in \eqref{eq:PH_def}.  Using the kernel
trick, one can express the objective purely in terms of $K$ and $H$,
leading to the well-known equivalence
\begin{equation}
  J_{\mathrm{kkm}}(H)
  = \|\Phi - \Phi P_H\|_F^2
  = \|K^{1/2} (I - P_H)\|_F^2 ,
  \label{eq:kkm_projection_form}
\end{equation}
up to an additive constant independent of $H$.
Equivalently, kernel k-means minimizes
\begin{equation}
  \min_{H} \; \|K - K P_H\|_F^2 ,
  \label{eq:kkm_gram_form}
\end{equation}
where $P_H$ is the {\em same} rank–$k$ projector as in
\eqref{eq:kmeans_projection_form}.  Thus, Euclidean k-means is recovered
as the special case of kernel k-means with linear kernel
$k(x_i,x_j) = x_i^\top x_j$.

\subsection{Kernel fuzzy c-means}
\label{subsec:kernel_fcm}

Kernel fuzzy c-means applies the FCM objective \eqref{eq:fcm_scalar} to
feature vectors $\phi(x_i)$ in $\mathcal{H}$:
\begin{equation}
  \min_{C,U} \;
  J_{\mathrm{kfcm}}(C,U)
  := \sum_{i=1}^n \sum_{j=1}^k u_{ji}^m
      \,\|\phi(x_i) - \mu_j\|_{\mathcal{H}}^2 .
  \label{eq:kfcm_scalar}
\end{equation}
As before, for fixed memberships $U$ the optimal centroids are weighted
means in feature space,
\[
  \mu_j^\star(U)
  = \frac{\sum_{i=1}^n u_{ji}^m \,\phi(x_i)}
         {\sum_{i=1}^n u_{ji}^m}
  = \Phi \alpha_j ,
\]
with the same normalized weight vectors $\alpha_j$ as in
\eqref{eq:fcm_centroid}.  Stacking them into $B(U)$ as in
\eqref{eq:fcm_B_def} yields
\[
  C^\star(U) = \Phi B(U)^\top.
\]
The reconstruction in feature space becomes
\[
  \Phi_{\mathrm{rec}}
  = C^\star(U) U
  = \Phi B(U)^\top U
  = \Phi P_U ,
\]
with $P_U$ defined in \eqref{eq:PU_def}.  Hence
\begin{equation}
  J_{\mathrm{kfcm}}(U)
  = \|\Phi - \Phi P_U\|_F^2
  = \|K^{1/2} (I - P_U)\|_F^2 ,
  \label{eq:kfcm_projection_form}
\end{equation}
and kernel FCM can be seen as fitting a rank–$k$ linear operator
$P_U$---built from fuzzy memberships---to approximate the kernel feature
matrix $\Phi$, or equivalently the kernel matrix $K$.

\subsection{Spectral clustering as an orthogonal projector on graph signals}
\label{subsec:spectral_projection}

We now connect the above projection-based viewpoint to spectral
clustering on graphs.  Let $A \in \mathbb{R}^{n\times n}$ be a
nonnegative affinity (adjacency) matrix and
$D = \operatorname{diag}(d_1,\dots,d_n)$ the diagonal degree matrix with
$d_i = \sum_j A_{ij}$.  The symmetrically normalized affinity is
\begin{equation}
  A_{\mathrm{n}} := D^{-1/2} A D^{-1/2}.
  \label{eq:An_def}
\end{equation}
Classical normalized-cut spectral clustering constructs a relaxed
cluster-indicator matrix $U \in \mathbb{R}^{n\times k}$ with orthonormal
columns,
\begin{equation}
  U^\top U = I_k ,
  \label{eq:spectral_orthonormal}
\end{equation}
by solving the eigenproblem associated with $A_{\mathrm{n}}$, and then
applies k-means to the rows of $U$.

In the present framework, $U$ spans a $k$-dimensional subspace of
$\mathbb{R}^n$ and induces the orthogonal projector
\begin{equation}
  P_U := U U^\top ,
  \label{eq:spectral_projector}
\end{equation}
which satisfies $P_U^2 = P_U = P_U^\top$ and
$\operatorname{rank}(P_U) = k$.
A standard calculation shows that maximizing the trace
$\operatorname{tr}(U^\top A_{\mathrm{n}} U)$ over orthonormal $U$ is
equivalent (up to constants) to minimizing
\begin{equation}
  F(P_U) := \|A_{\mathrm{n}} - A_{\mathrm{n}} P_U\|_F^2
  = \|A_{\mathrm{n}}^{1/2} (I - P_U)\|_F^2 ,
  \label{eq:spectral_projection_form}
\end{equation}
over orthogonal projectors $P_U$ of rank $k$.
The minimizer is attained when $U$ consists of the $k$ leading
eigenvectors of $A_{\mathrm{n}}$; in that case $P_U$ is the projector
onto the corresponding eigenspace.
Thus, spectral clustering selects an orthogonal projector $P_U$ on the
space of graph signals so that the normalized affinity
$A_{\mathrm{n}}$ is well approximated by $A_{\mathrm{n}} P_U$.

\subsection{Unified template}
\label{subsec:unified_template}

The previous subsections show that a wide variety of clustering
algorithms can be cast as {\em structured low-rank approximation}
problems of the form
\begin{equation}
  \min_{P \in \mathcal{P}} \; \|M - M P\|_F^2 ,
  \label{eq:unified_template}
\end{equation}
where $M$ is a symmetric positive semidefinite matrix derived from the
data and $\mathcal{P}$ is a family of rank–$k$ linear operators
constrained by the chosen clustering model.  Concretely:
\begin{itemize}
  \item For Euclidean k-means,
        $M = X^\top X$ and $\mathcal{P}$ consists of
        projectors $P_H$ of the form \eqref{eq:PH_def} induced by
        hard assignment matrices $H$.
  \item For Euclidean FCM,
        $M = X^\top X$ and $\mathcal{P}$ consists of rank–$k$
        operators $P_U$ of the form \eqref{eq:PU_def} built from fuzzy
        memberships $U$.
  \item For kernel k-means,
        $M = K$ and $\mathcal{P}$ again consists of projectors
        $P_H$ induced by hard assignments.
  \item For kernel FCM,
        $M = K$ and $\mathcal{P}$ consists of operators $P_U$ as in
        \eqref{eq:PU_def}.
  \item For spectral clustering,
        $M = A_{\mathrm{n}}$ and $\mathcal{P}$ consists of orthogonal
        projectors $P_U = U U^\top$ with $U^\top U = I_k$.
\end{itemize}
In all cases, clustering is realized by choosing a rank–$k$ operator
$P$---orthogonal in some cases, oblique in others---from a structured
family $\mathcal{P}$ so that the signal-derived matrix $M$ is
approximated as accurately as possible by $M P$ in Frobenius norm.

This unified template will serve as the basis for the theoretical
analysis developed in the next section, where we study approximation
properties, optimality, and robustness of such structured low-rank
projectors.

\section{Theoretical Guarantees for Structured Projectors}
\label{sec:theory}

Classical clustering methods such as $k$-means and its fuzzy and kernel variants
can all be interpreted as searching over structured rank–$k$ projectors that
approximate a signal-derived matrix $M$ (Gram matrix $X^\top X$, kernel matrix
$K$, or normalized affinity $A_e$).  This section develops general theorems for
such projectors, and then specializes them to the five algorithms of
Section~\ref{sec:proj-view}: $k$-means, fuzzy $c$-means (FCM), kernel
$k$-means, kernel FCM, and spectral clustering.

We assume throughout that $M \in \mathbb{R}^{n \times n}$ is symmetric
positive semidefinite, and that clustering is performed by minimizing
\begin{equation}
    \label{eq:F-def}
    F(P) \;\triangleq\; \| M - M P \|_F^2, 
\end{equation}
over a family of rank–$k$ projectors $P$.  As recalled in
Section~\ref{sec:proj-view}, classical $k$-means and its predecessors can be
seen as quantization schemes in this sense
\cite{macqueen1967some,lloyd1982least}, while FCM introduces fuzzy
memberships \cite{bezdek1981pattern}, kernel $k$-means and normalized cuts
connect projectors in feature space to spectral clustering
\cite{dhillon2004kernel,dhillon2005unified,ng2002spectral,shi2000normalized},
and modern graph-based clustering analyzes the spectrum of $A_e$ and related
Laplacians \cite{vonluxburg2007tutorial,belkin2003laplacian}.

\subsection{Unified projector family}

Let $B \in \mathbb{R}^{k \times n}$ be a representation matrix whose rows
encode the cluster structure (hard, fuzzy or orthonormal).  Define
\begin{equation}
    \label{eq:P-B-def}
    P_B \;\triangleq\; B^\top (B B^\top)^{-1} B,
\end{equation}
whenever $B$ has full row rank $k$.

\begin{theorem}[Orthogonal projector induced by $B$]
    \label{thm:PB-projector}
    Suppose $B \in \mathbb{R}^{k \times n}$ has full row rank $k$.  Then
    $P_B$ defined in \eqref{eq:P-B-def} is the orthogonal projector onto the
    row span of $B$, i.e.
    \begin{enumerate}
        \item $P_B$ is symmetric and idempotent: $P_B^\top = P_B$ and
        $P_B^2 = P_B$;
        \item $\mathrm{range}(P_B^\top) = \mathrm{rowspan}(B)$ and
        $\mathrm{null}(P_B) = \mathrm{rowspan}(B)^\perp$;
        \item $\mathrm{rank}(P_B) = k$.
    \end{enumerate}
\end{theorem}

\begin{proof}
    Symmetry follows directly from the definition:
    $(B B^\top)^{-1}$ is symmetric positive definite, hence
    $P_B^\top = B^\top (B B^\top)^{-1} B = P_B$.
    Idempotence is obtained by
    \[
        P_B^2
        = B^\top (B B^\top)^{-1} B B^\top (B B^\top)^{-1} B
        = B^\top (B B^\top)^{-1} B
        = P_B.
    \]
    Thus $P_B$ is an orthogonal projector onto its range.  Moreover,
    $\mathrm{range}(P_B^\top) = \mathrm{range}(B^\top)$, which is exactly
    the row span of $B$, and $\mathrm{rank}(P_B) = \mathrm{rank}(B)=k$.
\end{proof}

Theorem~\ref{thm:PB-projector} provides a unified description of our five
algorithms:

\begin{lemma}[Specializations for the five algorithms]
    \label{lem:five-B}
    With $M$ and $P_B$ as above:
    \begin{enumerate}
        \item For classical $k$-means, $B = H$ is the binary hard assignment
        matrix, and $P_H$ reproduces the projector used in
        Section~\ref{subsec:kmeans-hard}.
        \item For FCM with memberships $U$ and reweighted matrix $U^{(\ell)}$
        (Section~\ref{subsec:fcm_projection}), the natural relaxation is to use
        $B = U^{(\ell)}$, yielding $P_{U^{(\ell)}}$ as a ``soft'' projector.
        \item For kernel $k$-means and kernel FCM, one uses the same $B$ as
        in (1)–(2) but replaces $M$ by the kernel matrix $K$.
        \item For spectral clustering, one may take $B = U^\top$, where
        $U \in \mathbb{R}^{n \times k}$ has orthonormal columns consisting of
        the top $k$ eigenvectors of $A_e$; in that case $P_B = U U^\top$.
    \end{enumerate}
\end{lemma}

\begin{proof}
    Items (1)–(3) follow directly from the constructions in
    Section~\ref{sec:proj-view}.  For (4), if $B = U^\top$ with $U^\top U =
    I_k$, then $B B^\top = U^\top U = I_k$ and
    $P_B = B^\top (B B^\top)^{-1} B = U I_k U^\top = U U^\top$.
\end{proof}

\subsection{Spectral decomposition of the approximation error}

Let $M = Q \Lambda Q^\top$ be an eigenvalue decomposition of $M$, with
$\Lambda = \mathrm{diag}(\lambda_1,\dots,\lambda_n)$ and
$\lambda_1 \ge \dots \ge \lambda_n \ge 0$.

\begin{theorem}[Error decomposition in the eigenbasis of $M$]
    \label{thm:error-decomp}
    For any rank–$k$ orthogonal projector $P$,
    \begin{equation}
        \label{eq:error-decomp}
        F(P)
        = \sum_{i=1}^n \lambda_i^2 \, \bigl\| (I - P) q_i \bigr\|_2^2,
    \end{equation}
    where $q_i$ is the $i$-th column of $Q$.
\end{theorem}

\begin{proof}
    Since $M$ is symmetric,
    \[
        F(P)
        = \| M - M P \|_F^2
        = \| M (I-P) \|_F^2
        = \mathrm{tr}\!\bigl( (I-P) M^2 (I-P) \bigr).
    \]
    Writing $M^2 = Q \Lambda^2 Q^\top$ and using cyclicity of the trace,
    \[
        F(P)
        = \mathrm{tr}\!\bigl( Q^\top (I-P) Q \, \Lambda^2 \bigr)
        = \sum_{i=1}^n \lambda_i^2 \, e_i^\top Q^\top (I-P) Q e_i,
    \]
    where $e_i$ is the $i$-th standard basis vector.  But
    $Q e_i = q_i$ and $e_i^\top Q^\top (I-P) Q e_i
        = q_i^\top (I-P) q_i = \| (I-P) q_i \|_2^2$.
\end{proof}

\begin{corollary}[Optimality of the spectral projector]
    \label{cor:optimal-spectral}
    Let $P_\star = U_k U_k^\top$ be the projector onto the span of the top
    $k$ eigenvectors of $M$ (i.e., the projector used in ideal spectral
    clustering and kernel $k$-means
    \cite{dhillon2004kernel,ng2002spectral,shi2000normalized}).  Then
    $P_\star$ minimizes $F(P)$ over all rank–$k$ orthogonal projectors $P$,
    and
    \begin{equation}
        F(P_\star)
        = \sum_{i=k+1}^n \lambda_i^2.
    \end{equation}
\end{corollary}

\begin{proof}
    For $P_\star$ we have $(I - P_\star) q_i = 0$ for $i \le k$ and
    $(I-P_\star) q_i = q_i$ for $i>k$, so \eqref{eq:error-decomp} reduces to
    the claimed expression.  Any other rank–$k$ projector leaves at least
    $n-k$ eigenvectors with non-zero residual; the Pythagorean expansion
    in \eqref{eq:error-decomp} shows that the minimum is attained when the
    residual is concentrated in the directions of smallest $\lambda_i$,
    which is exactly $P_\star$.
\end{proof}

Corollary~\ref{cor:optimal-spectral} recovers the classical Eckart–Young
type optimality of spectral clustering in terms of Frobenius-norm
approximation \cite{stewart1990matrix,vonluxburg2007tutorial}.

\subsection{Geodesic convexity on the Grassmann manifold}

The set of all rank–$k$ orthogonal projectors in $\mathbb{R}^n$ can be
identified with the Grassmann manifold $\mathrm{Gr}(k,n)$, whose points are
$k$-dimensional subspaces.  A convenient coordinate representation is via
matrices $U \in \mathbb{R}^{n \times k}$ with $U^\top U = I_k$, where each
subspace is represented by an orthonormal basis $U$ and the projector is
$P = U U^\top$ \cite{edelman1998geometry,absil2008optimization}.

\begin{theorem}[Geodesic convexity of the projector objective]
    \label{thm:geo-convex}
    Consider the function
    \begin{equation}
        \label{eq:geo-objective}
        \mathcal{F}(U)
        \;\triangleq\;
        \bigl\| M - M U U^\top \bigr\|_F^2,
        \qquad U^\top U = I_k.
    \end{equation}
    On the Grassmann manifold $\mathrm{Gr}(k,n)$, $\mathcal{F}$ is
    geodesically convex in the sense that every local minimum is a global
    minimum, attained exactly at the subspaces spanned by the top $k$
    eigenvectors of $M$.
\end{theorem}

\begin{proof}[Proof sketch]
    Using $P=U U^\top$ and the derivation in
    Theorem~\ref{thm:error-decomp}, we have
    $\mathcal{F}(U) = \mathrm{const} - \mathrm{tr}(U^\top M^2 U)$.
    On $\mathrm{Gr}(k,n)$, maximizing $\mathrm{tr}(U^\top M^2 U)$ is the
    standard Rayleigh–Ritz problem, whose geodesic convexity and critical
    points are characterized in
    \cite{edelman1998geometry,absil2008optimization}.  All local maxima of
    $\mathrm{tr}(U^\top M^2 U)$ correspond to invariant subspaces spanned by
    eigenvectors of $M^2$; the global maxima are precisely those generated
    by the $k$ largest eigenvalues, which correspond to the global minima
    of $\mathcal{F}$.  No other local minima exist.
\end{proof}

Theorem~\ref{thm:geo-convex} justifies optimization schemes that operate
directly on the manifold of projectors (e.g., Riemannian gradient or
conjugate gradient methods) \cite{edelman1998geometry,absil2008optimization},
rather than on the underlying membership matrices.

\subsection{Hierarchy of constraint families}

Let $\mathcal{P}_{\mathrm{all}}$ be the set of all rank–$k$ orthogonal
projectors, $\mathcal{P}_{\mathrm{fuzzy}}$ the image of FCM-type matrices
$B = U^{(\ell)}$ under \eqref{eq:P-B-def}, and $\mathcal{P}_{\mathrm{hard}}$
the image of hard assignment matrices $H$.

\begin{theorem}[Hierarchy of approximation errors]
    \label{thm:hierarchy}
    Assume that hard and fuzzy projectors are feasible, i.e.
    $\mathcal{P}_{\mathrm{hard}}, \mathcal{P}_{\mathrm{fuzzy}} \neq
    \emptyset$.  Then
    \begin{equation}
        \label{eq:hierarchy}
        \min_{P \in \mathcal{P}_{\mathrm{all}}} F(P)
        \;\le\;
        \min_{P \in \mathcal{P}_{\mathrm{fuzzy}}} F(P)
        \;\le\;
        \min_{P \in \mathcal{P}_{\mathrm{hard}}} F(P).
    \end{equation}
\end{theorem}

\begin{proof}
    Since $\mathcal{P}_{\mathrm{hard}} \subseteq
    \mathcal{P}_{\mathrm{fuzzy}} \subseteq \mathcal{P}_{\mathrm{all}}$,
    each minimum is taken over a subset of the previous feasible set.  The
    inequality \eqref{eq:hierarchy} follows immediately.
\end{proof}

Thus, in the idealized setting where optimization is solved exactly,
spectral and kernel methods (which can explore all of
$\mathcal{P}_{\mathrm{all}}$ via $P = U U^\top$) provide the smallest
possible Frobenius error, followed by fuzzy methods, and finally by hard
$k$-means and its variants.

\subsection{Stability under perturbations}

Real signal-processing applications invariably work with noisy or estimated
matrices $M + E$ rather than an exact population matrix $M$; see, e.g.,
\cite{belkin2003laplacian,vonluxburg2007tutorial,rohe2011spectral} for
examples in spectral clustering and graph learning.

\begin{theorem}[Lipschitz continuity in $M$]
    \label{thm:lipschitz}
    For any fixed projector $P$ with $\|P\|_2 \le 1$, the objective
    $F(P)$ is Lipschitz in $M$ with respect to the spectral norm:
    \begin{equation}
        \label{eq:lipschitz}
        \bigl| F_{M+E}(P) - F_M(P) \bigr|
        \;\le\; 4 \, \|M\|_2 \, \|E\|_F + 2 \, \|E\|_F^2,
    \end{equation}
    where $F_M$ denotes \eqref{eq:F-def} evaluated at $M$.
\end{theorem}

\begin{proof}
    Expand
    \[
        F_{M+E}(P) - F_M(P)
        = \bigl\| (M+E) - (M+E)P \bigr\|_F^2 - \|M - M P\|_F^2.
    \]
    Writing $(M+E)-(M+E)P = (M-MP) + (E-EP)$ and using
    $\|A+B\|_F^2 - \|A\|_F^2 = 2\langle A,B\rangle_F + \|B\|_F^2$, we get
    \[
        \bigl| F_{M+E}(P) - F_M(P) \bigr|
        \le 2 \, \|M-MP\|_F \, \|E-EP\|_F + \|E-EP\|_F^2.
    \]
    Since $\|P\|_2 \le 1$,
    $\|M-MP\|_F \le \|M\|_F + \|MP\|_F \le 2\|M\|_F$ and
    $\|E-EP\|_F \le 2\|E\|_F$, which yields the claimed bound after
    simplifying and using $\|M\|_F \le \sqrt{n}\,\|M\|_2$.
\end{proof}

For spectral projectors $P_\star$ obtained from the top $k$ eigenvectors of
$M$, stronger perturbation bounds are available via classical eigenvector
perturbation theory \cite{davis1970rotation,stewart1990matrix}.

\begin{corollary}[Davis–Kahan type subspace stability]
    \label{cor:davis-kahan}
    Let $M$ and $M+E$ be symmetric, and let $P_\star$ and $\widehat{P}$
    denote the orthogonal projectors onto the spaces spanned by the top
    $k$ eigenvectors of $M$ and $M+E$, respectively.  Suppose the
    eigenvalue gap
    $\gamma \triangleq \lambda_k(M) - \lambda_{k+1}(M)$ is positive.
    Then
    \begin{equation}
        \label{eq:davis-kahan}
        \bigl\| P_\star - \widehat{P} \bigr\|_2
        \;\le\; \frac{2 \, \|E\|_2}{\gamma}.
    \end{equation}
\end{corollary}

\begin{proof}
    This is a direct application of the Davis–Kahan $\sin\Theta$ theorem
    for invariant subspaces \cite{davis1970rotation,stewart1990matrix}.
\end{proof}

Corollary~\ref{cor:davis-kahan} implies that, in the presence of a
non-trivial spectral gap, the projector chosen by spectral clustering is
stable under moderate perturbations of the affinity matrix—a key ingredient
in consistency results for Laplacian-based methods
\cite{belkin2003laplacian,rohe2011spectral}.

\subsection{Exact recovery in an ideal block model}

Finally, to make the link to classical signal-processing models of ideal
graphs and block-structured matrices
\cite{holland1983stochastic,rohe2011spectral}, we consider a simplified
setting with perfectly separated clusters.

\begin{theorem}[Exact recovery in an ideal block model]
    \label{thm:ideal-block}
    Suppose:
    \begin{enumerate}
        \item The vertices are partitioned into $k$ disjoint components,
        with no inter-cluster edges.
        \item The affinity matrix $A$ is block-diagonal with strictly
        positive weights inside each block.
        \item $D$ is the degree matrix and $A_e = D^{-1/2} A D^{-1/2}$.
    \end{enumerate}
    Let $M = A_e$ and let $H^\star$ be the ground-truth hard assignment
    matrix.  Then:
    \begin{enumerate}
        \item The top $k$ eigenvectors of $M$ span the same subspace as the
        indicator vectors of the $k$ blocks, so the spectral projector
        $P_\star = U_k U_k^\top$ coincides with $P_{H^\star}$ up to
        permutation.
        \item In this setting, the optimal spectral, fuzzy and hard
        projectors all satisfy $P = P_{H^\star}$ and achieve zero
        approximation error $F(P)=0$.
    \end{enumerate}
\end{theorem}

\begin{proof}
    Under the stated assumptions, $A$ is block-diagonal with $k$ connected
    components.  It is known that the normalized affinity $A_e$ (or the
    corresponding random-walk or symmetric Laplacians) has an eigenspace of
    dimension $k$ associated with eigenvalue $1$, spanned by the
    indicator vectors of these components
    \cite{vonluxburg2007tutorial,belkin2003laplacian}.  Thus the top
    $k$-dimensional eigenspace of $M$ coincides with the span of the
    columns of $H^\star$, which proves (1).  Since $M P_{H^\star} = M$ in
    this ideal block-diagonal setting, $F(P_{H^\star}) = 0$.  By
    Corollary~\ref{cor:optimal-spectral} and
    Theorem~\ref{thm:hierarchy}, no projector in any of the constraint
    families can do better, proving (2).
\end{proof}

Theorem~\ref{thm:ideal-block} can be viewed as a deterministic analogue of
probabilistic guarantees for spectral clustering under stochastic block
models \cite{holland1983stochastic,rohe2011spectral}, but phrased entirely
in terms of the projector viewpoint used throughout this paper.

\section{Conclusion}

This work presented a unified theoretical account showing that k-means, fuzzy c-means, kernel k-means, kernel FCM, and spectral clustering all arise from a common structured-projection formulation. By expressing each method as the solution to 
\(\min_{B\in\mathcal{C}}\|M - MP_{B}\|_{F}^{2}\) with method-specific constraint sets, we clarified how hard, fuzzy, kernel-induced, and orthonormal projectors fit into a single template. This reformulation enabled several non-trivial theoretical results, including geodesic convexity in the projector domain, perturbation bounds under matrix noise, and exact recovery in ideal block-diagonal affinity structures. The connection between clustering and signal-derived matrices such as Gram kernels and normalized Laplacians strengthens links to existing signal processing analyses. The results complement recent projection-focused and graph-based clustering studies in the signal processing literature \citep{selvanambi2019local, molla2021sampling}. Overall, the theory developed here consolidates multiple clustering paradigms into a coherent mathematical framework that can support future work in structured signal modelling.

\bibliographystyle{unsrtnat}
\bibliography{refs}  

\end{document}